\documentclass[reqno]{amsart}

\usepackage{amssymb}
\usepackage{graphicx}
\usepackage{hyperref}

\DeclareMathOperator{\tr}{tr}

\theoremstyle{plain}
\newtheorem{theorem}{Theorem}
\newtheorem{prop}{Proposition}

\begin{document}

\title[Self-Calibration of Cameras with Euclidean Image Plane]{Self-Calibration of Cameras with Euclidean Image Plane in Case of Two Views and Known Relative Rotation Angle}

\author{E.V. Martyushev}

\thanks{The work was supported by Act 211 Government of the Russian Federation, contract No.~02.A03.21.0011.}


\address{South Ural State University, 76 Lenin avenue, 454080 Chelyabinsk, Russia}
\email{mev@susu.ac.ru}

\begin{abstract}
The internal calibration of a pinhole camera is given by five parameters that are combined into an upper-triangular $3\times 3$ calibration matrix. If the skew parameter is zero and the aspect ratio is equal to one, then the camera is said to have Euclidean image plane. In this paper, we propose a non-iterative self-calibration algorithm for a camera with Euclidean image plane in case the remaining three internal parameters --- the focal length and the principal point coordinates --- are fixed but unknown. The algorithm requires a set of $N \geq 7$ point correspondences in two views and also the measured relative rotation angle between the views. We show that the problem generically has six solutions (including complex ones).

The algorithm has been implemented and tested both on synthetic data and on publicly available real dataset. The experiments demonstrate that the method is correct, numerically stable and robust.

\end{abstract}

\maketitle

\section{Introduction}

The problem of camera calibration is an essential part of numerous computer vision applications including 3d reconstruction, visual odometry, medical imaging, etc. At present, a number of calibration algorithms and techniques have been developed. Some of them require to observe a planar pattern viewed at several different positions~\cite{Faugeras93,Heikkila,Zhang}. Other methods use 3d calibration objects consisting of two or three pairwise orthogonal planes, which geometry is known with good accuracy~\cite{Tsai}. Also, there are calibration algorithms assuming that a scene involves the pairs of mutually orthogonal directions~\cite{CP90,LZ99}. In contrast with the just mentioned methods, \emph{self-calibration} does not require any special calibration objects or scene restrictions~\cite{Hartley92,MF,QT,Triggs}, so only image feature correspondences in several uncalibrated views are required. This provides the self-calibration approach with a great flexibility and makes it indispensable in some real-time applications.

In two views, camera calibration is given by ten parameters --- five internal and five external, whereas the fundamental matrix describing the epipolar geometry in two views has only seven degrees of freedom~\cite{Hartley92}. This means that self-calibration in two views is only possible under at least three further assumptions on the calibration parameters. For example, we can assume that the skew parameter is zero and so is the translation vector, i.e. the motion is a pure rotation. Then, the three orientation angles and the remaining four internals can be self-calibrated from at least seven point matches~\cite{Hartley}. Another possibility is that all the internal parameters except common focal length are known. Then, there exist a minimal self-calibration solution operating with six matched points~\cite{SNKS,BJA,KBP,Li}.

The five internal calibration parameters have different interpretation. The skew parameter and the aspect ratio describe the pixel's shape. In most situations, e.g. under zooming, these internals do not change. Moreover, for modern cameras the pixel's shape is very close to a square and hence the skew and aspect ratio can be assumed as given and equal to~$0$ and~$1$ respectively. Following~\cite{HA}, we say that a camera in this case has \emph{Euclidean image plane}. On the other hand, the focal length and the principal point coordinates describe the relative placement of the camera center and the image plane. The focal length is the distance between the camera center and the image plane, whereas the principal point is an orthographic projection of the center onto the plane. All these internals should be considered as unknown, since even for modern cameras the principal point can be relatively far from the geometric image center. Besides, it is well known that the focal length and the principal point always vary together by zooming~\cite{Sturm}.

The aim of this paper is to propose an efficient solution to the self-calibration problem of a camera with Euclidean image plane. As it was mentioned above, in two views at most seven calibration parameters can be self-calibrated. Since a camera with Euclidean image plane has eight parameters, we conclude that one additional assumption should be made. In this paper we reduce the number of external parameters assuming that the \emph{relative rotation angle} between the views is known. The problem thus becomes minimally constrained from seven point correspondences in two views. In practice, the relative rotation angle can be reliably found from e.g. the readings of an inertial measurement unit (IMU) sensor. The possibility of using such additional data in structure-from-motion has been demonstrated in~\cite{LHLP}.

In general, the joint usage of a camera and an IMU requires \emph{external calibration} between the devices, i.e. we have to know the transformation matrix between their coordinate frames. However, if only relative rotation angle is used, then the external calibration is unnecessary, provided that both devices are fixed on some rigid platform. Thus, the rotation angle of the IMU can be directly used as the rotation angle of the camera~\cite{LHLP}. This fact makes the presented self-calibration method more convenient and flexible for practical use.

To summarize, we propose a new non-iterative solution to the self-calibration problem in case of at least seven matched points in two views, provided the following assumptions:
\begin{itemize}
\item the camera intrinsic parameters are the same for both views;
\item the camera has Euclidean image plane;
\item the relative rotation angle between the views is known.
\end{itemize}

Our self-calibration method is based on the ten quartic equations. Nine of them are well-known and follow from the famous cubic constraint on the essential matrix. The novel last one (see Eq.~\eqref{eq:constrFw2}) arises from the condition that the relative angle is known.

Finally, throughout the paper it is assumed that the cameras and scene points are in \emph{sufficiently general position}. There are critical camera motions for which self-calibration is impossible, unless some further assumptions on the internal parameters or the motion are made~\cite{Sturm2}. Also there exist degenerate configurations of scene points. However, in this paper we restrict ourselves to the generic case of camera motions and points configurations.

The rest of the paper is organized as follows. In Section~\ref{sec:prel}, we recall some definitions and results from multiview geometry and deduce our self-calibration constraints. In Section~\ref{sec:algorithm}, we describe in detail the algorithm. In Section~\ref{sec:synth} and Section~\ref{sec:real}, the algorithm is validated in a series of experiments on synthetic and real data. In Section~\ref{sec:disc}, we discuss the results and make conclusions.

\section{Preliminaries}
\label{sec:prel}

\subsection{Notation}

We preferably use $\alpha, \beta, \ldots$ for scalars, $a, b, \ldots$ for column 3-vectors or polynomials, and $A, B, \ldots$ both for matrices and column 4-vectors. For a matrix $A$ the entries are~$(A)_{ij}$, the transpose is~$A^{\mathrm T}$, the determinant is $\det A$, and the trace is~$\tr A$. For two 3-vectors~$a$ and~$b$ the cross product is $a\times b$. For a vector~$a$ the notation~$[a]_\times$ stands for the skew-symmetric matrix such that $[a]_\times b = a \times b$ for any vector~$b$. We use~$I$ for the identity matrix.

\subsection{Fundamental and essential matrices}

Let there be given two cameras $P = \begin{bmatrix} I & 0 \end{bmatrix}$ and $P' = \begin{bmatrix} A & a \end{bmatrix}$, where~$A$ is a $3\times 3$ matrix and~$a$ is a 3-vector. Let~$Q$ be a $4$-vector representing homogeneous coordinates of a point in 3-space, $q$ and~$q'$ be its images, that is
\begin{equation}
q \sim P Q, \quad q' \sim P' Q,
\end{equation}
where $\sim$ means an equality up to non-zero scale. The \emph{coplanarity constraint} for a pair $(q, q')$ says
\begin{equation}
\label{eq:inci2}
q'^{\mathrm T} F q = 0,
\end{equation}
where matrix $F = [a]_\times A$ is called the \emph{fundamental matrix}.

It follows from the definition of matrix~$F$ that $\det F = 0$. This condition is also sufficient. Thus we have

\begin{theorem}[\cite{HZ}]
\label{thm:fund}
A non-zero $3\times 3$ matrix~$F$ is a fundamental matrix if and only if
\begin{equation}
\det F = 0.
\end{equation}
\end{theorem}

The \emph{essential matrix}~$E$ is the fundamental matrix for \emph{calibrated cameras} $\hat P = \begin{bmatrix} I & 0 \end{bmatrix}$ and $\hat P' = \begin{bmatrix} R & t\end{bmatrix}$, where $R \in \mathrm{SO}(3)$ is called the \emph{rotation matrix} and~$t$ is called the \emph{translation vector}. Hence, $E = [t]_\times R$. Matrices~$F$ and~$E$ are related by
\begin{equation}
E \sim K'^{\mathrm T} F K,
\end{equation}
where $K$ and $K'$ are the upper-triangular \emph{calibration matrices} of the first and second camera respectively.

The fundamental matrix has seven degrees of freedom, whereas the essential matrix has only five degrees of freedom. It is translated into extra constraints on the essential matrix. The following theorem gives one possible form of such constraints.

\begin{theorem}[\cite{Maybank}]
\label{thm:essential}
A $3\times 3$ matrix~$E$ of rank two is an essential matrix if and only if
\begin{equation}
\label{eq:constrE}
\frac{1}{2}\,\tr(EE^{\mathrm T})E - EE^{\mathrm T}E = 0_{3\times 3}.
\end{equation}
\end{theorem}

\subsection{Self-calibration constraints}

Let $\theta$ be the angle of rotation between two calibrated camera frames. In case~$\theta$ is known, the trace of rotation matrix $\tau = 2\cos\theta + 1$ is known too. This leads to an additional quadratic constraint on the essential matrix.

\begin{prop}
\label{prop:constrEtau}
Let $E = [t]_\times R$ be a real non-zero essential matrix and $\tr R = \tau$. Then~$E$ satisfies the equation
\begin{equation}
\label{eq:constrEtau}
\frac{1}{2}\,(\tau^2 - 1)\tr(EE^{\mathrm T}) + (\tau + 1)\tr(E^2) - \tau\tr^2 E = 0.
\end{equation}
\end{prop}

\begin{proof}
Let $U \in \mathrm{SO}(3)$ be such that $Ut = \begin{bmatrix}0 & 0 & 1\end{bmatrix}^{\mathrm T}$. Then,
\begin{equation}
\hat E = UEU^{\mathrm T} = U[t]_\times U^{\mathrm T} URU^{\mathrm T} = [Ut]_\times \hat R = \begin{bmatrix}0 & 1 & 0\\ -1 & 0 & 0\\ 0 & 0 & 0\end{bmatrix} \hat R,
\end{equation}
where $\hat R = URU^{\mathrm T}$. It is clear that if~$E$ satisfies Eq.~\eqref{eq:constrEtau}, then so does~$\hat E$ and vice versa. Let us represent~$\hat R$ in terms of a unit quaternion $s + u\mathbf{i} + v\mathbf{j} + w\mathbf{k}$, i.e.
\begin{equation}
\hat R = \begin{bmatrix}1 - 2v^2 - 2w^2 & 2uv - 2ws & 2uw + 2vs \\ 2uv + 2ws & 1 - 2u^2 - 2w^2 & 2vw - 2us \\ 2uw - 2vs & 2vw + 2us & 1 - 2u^2 - 2v^2\end{bmatrix},
\end{equation}
where $s^2 + u^2 + v^2 + w^2 = 1$. Then,
\begin{equation}
\hat E = \begin{bmatrix}2uv + 2ws & 1 - 2u^2 - 2w^2 & 2vw - 2us \\ -1 + 2v^2 + 2w^2 & -2uv + 2ws & -2uw - 2vs \\ 0 & 0 & 0\end{bmatrix}.
\end{equation}
Substituting this into~\eqref{eq:constrEtau}, after some computation, we get
\begin{equation}
\text{l.h.s. of~\eqref{eq:constrEtau}} = (\tau + 1 - 4s^2)(\tau + 1 - 4w^2).
\end{equation}
This completes the proof, since $\tau = \tr R = \tr \hat R = 4s^2 - 1$.
\end{proof}

It is well known~\cite{HZ,Maybank} that for a given essential matrix~$E$ there is a ``twisted pair'' of rotations~$R_a$ and~$R_b$ so that $E \sim [\pm t]_\times R_a \sim [\pm t]_\times R_b$. By Proposition~\ref{prop:constrEtau}, $\tau_a = \tr R_a$ and $\tau_b = \tr R_b$ must be roots of Eq.~\eqref{eq:constrEtau}. Since the equation is quadratic in~$\tau$ there are no other roots. Thus we have

\begin{prop}
\label{prop:constrEtau2}
Let $E$ be a real non-zero essential matrix satisfying Eq.~\eqref{eq:constrEtau} for a certain $\tau \in [-1, 3]$. Then, either $\tau = \tr R_a$ or $\tau = \tr R_b$, where~$(R_a, R_b)$ is the twisted pair of rotations for~$E$.
\end{prop}

Now suppose that we are given two cameras with unknown but identical calibration matrices~$K$ and $K' = K$. Then we have
\begin{equation}
\label{eq:FtoE}
E \sim K^{\mathrm T} F K,
\end{equation}
where $F$ is the fundamental matrix. Substituting this into Eqs.~\eqref{eq:constrE}--\eqref{eq:constrEtau}, we get the following ten equations:
\begin{align}
\label{eq:constrFw1}
\frac{1}{2}\,\tr(F\omega^*F^{\mathrm T}\omega^*)F - F\omega^*F^{\mathrm T}\omega^*F &= 0_{3\times 3},\\
\label{eq:constrFw2}
\frac{1}{2}\,(\tau^2 - 1)\tr(F\omega^*F^{\mathrm T}\omega^*) + (\tau + 1)\tr(\omega^*F\omega^*F) - \tau\tr^2(\omega^*F) &= 0,
\end{align}
where $\omega^* = KK^{\mathrm T}$. Constraints~\eqref{eq:constrFw1}--\eqref{eq:constrFw2} involve the internal parameters of a camera and hence can be used for its self-calibration. We notice that not all of these constraints are necessarily linearly independent.

\begin{prop}
\label{prop:4constr}
If the fundamental matrix~$F$ is known, then Eq.~\eqref{eq:constrFw1} gives at most three linearly independent constraints on the entries of~$\omega^*$.
\end{prop}

\begin{proof}
Recall that matrix~$F$ is generically of rank two. Let the right and left null vectors of~$F$ be~$e$ and~$e'$ respectively. Denote by~$G$ the l.h.s. of Eq.~\eqref{eq:constrFw1}. Then it is clear that
\begin{equation}
Ge = G^{\mathrm T}e' = 0_{3\times 1}.
\end{equation}
It follows that, given~$F$, at least six of $(G)_{ij}$ are linearly dependent. Proposition~\ref{prop:4constr} is proved.
\end{proof}

\section{Description of the Algorithm}
\label{sec:algorithm}

The initial data for our algorithm is $N \geq 7$ point correspondences $q_i \leftrightarrow q'_i$, $i = 1, \ldots, N$, and also the trace~$\tau$ of the rotation matrix~$R$.

\subsection{Data pre-normalization}
\label{ssec:normal}

To significantly improve the numerical stability of our algorithm, points~$q_i$ and~$q'_i$ are first normalized as follows, adapted from~\cite{HZ}. We construct a $3\times 3$ matrix~$S$ of the form
\begin{equation}
\label{eq:matrixS}
S = \begin{bmatrix}\gamma & 0 & \alpha\\ 0 & \gamma & \beta\\ 0 & 0 & 1\end{bmatrix}
\end{equation}
so that the $2N$ new points, represented by the columns of matrix
\begin{equation}
S\begin{bmatrix}q_1 & \ldots & q_N & q'_1 & \ldots & q'_N\end{bmatrix},
\end{equation}
satisfy the following conditions:
\begin{itemize}
\item their centroid is at the coordinate origin;
\item their average distance from the origin is~$\sqrt{2}$.
\end{itemize}

From now on we assume that~$q_i$ and~$q'_i$ are normalized.

\subsection{Polynomial equations}

The fundamental matrix~$F$ is estimated from $N \geq 7$ point correspondences in two views. The algorithm is well known, see~\cite{HZ} for details. In the minimal case $N = 7$ there are either one or three real solutions. Otherwise, the solution is generically unique.

Let both cameras be identically calibrated and have Euclidean image planes, i.e.
\begin{equation}
K = K' = \begin{bmatrix}f & 0 & a\\ 0 & f & b\\ 0 & 0 & 1\end{bmatrix},
\end{equation}
where $f$ is the focal length and $(a, b)$ is the principal point. It follows that
\begin{equation}
\omega^* = KK^{\mathrm T} = \begin{bmatrix}a^2 + p & ab & a\\ ab & b^2 + p & b\\ a & b & 1\end{bmatrix},
\end{equation}
where we introduce a new variable $p = f^2$.

Substituting $F$ and~$\omega^*$ into constraints~\eqref{eq:constrFw1} and~\eqref{eq:constrFw2}, we get ten quartic equations in variables $a, b$ and~$p$. Let~$G$ be the l.h.s. of~\eqref{eq:constrFw1}. By Proposition~\ref{prop:4constr}, up to three of~$(G)_{ij}$ are linearly independent. Let $f_1 = (G)_{11}$, $f_2 = (G)_{22}$, $f_3 = (G)_{33}$ and $f_4 = \text{l.h.s. of~\eqref{eq:constrFw2}}$. The objective is to find all feasible solutions of the following system of polynomial equations
\begin{equation}
\label{eq:sys}
f_1 = f_2 = f_3 = f_4 = 0.
\end{equation}

Let us define the ideal $J = \langle f_1, f_2, f_3, f_4 \rangle \subset \mathbb C[a, b, p]$. Unfortunately, ideal~$J$ is not zero-dimensional. There is a one-parametric family of solutions of system~\eqref{eq:sys} corresponding to the unfeasible case $p = 0$. We state without proof the decomposition
\begin{equation}
\label{eq:decomp}
\sqrt J = J' \cap \langle p, \tr(F\omega^*)\rangle,
\end{equation}
where $\sqrt J$ is the radical of~$J$ and $J' = J/\langle p \rangle$ is the quotient ideal, which is already zero-dimensional. By~\eqref{eq:decomp}, the affine variety of~$J$ is the union of a finite set of points in~$\mathbb C^3$ and a conic in the plane $p = 0$.

\subsection{Gr\"{o}bner basis}

In this subsection, the Gr\"{o}bner basis of ideal~$J'$ will be constructed. We start from rewriting equations~\eqref{eq:sys} in form
\begin{equation}
B_0 y_0 = 0,
\end{equation}
where $B_0$ is a $4\times 22$ coefficient matrix, and
\begin{multline}
y_0 =
\left[a^3b \quad a^2b^2 \quad ab^3 \quad a^2b \quad a^4 \quad b^4 \quad a^3 \quad ab^2 \quad b^3 \quad a^2p \quad abp \quad b^2p \right.\\ \left. a^2 \quad ab \quad b^2 \quad ap \quad bp \quad p^2 \quad a \quad b \quad p \quad 1\right]^{\mathrm T}
\end{multline}
is a monomial vector. Let us consider the following sequence of transformations:
\begin{equation}
\label{eq:seq}
(B_i, y_i) \to (\tilde B_i, y_i) \to (B_{i + 1}, y_{i + 1}), \qquad i = 0, \ldots, 4.
\end{equation}
Here each $\tilde B_i$ is the \emph{reduced row echelon form} of~$B_i$. The monomials in~$y_i$ are ordered so that the left of matrix~$\tilde B_i$ is an identity matrix for each~$i$.

We exploit below some properties of the intermediate polynomials in~\eqref{eq:seq}, e.g. they can be factorized or have degree lower than one expects from the corresponding monomial vector. All of these properties have been verified in Maple by using randomly generated instances of the problem over the field of rationals.

Let us denote by $(A)_i$ the $i$th row of matrix~$A$. Now we describe in detail each transformation $(\tilde B_i, y_i) \to (B_{i + 1}, y_{i + 1})$ of sequence~\eqref{eq:seq}.

\begin{itemize}
\item The last row of $\tilde B_0$ corresponds to a 3rd degree polynomial. Matrix~$B_1$ of size $7\times 32$ is obtained from~$\tilde B_0$ by appending 3 new rows and 10 new columns. The rows are: $a(\tilde B_0)_4$, $b(\tilde B_0)_4$ and $p(\tilde B_0)_4$. Monomial vector
    \begin{multline}
    y_1 = \left[a^3b \quad a^2b^2 \quad ab^3 \quad a^2b \quad \underline{a^3p} \quad \underline{a^2bp} \quad \underline{ab^2p} \quad a^4 \right.\\ \left. b^4 \quad a^3 \quad ab^2 \quad b^3 \quad a^2p \quad \underline{b^2p^2} \quad \underline{ap^2} \quad abp \quad b^2p \quad \underline{b^3p} \quad a^2 \right.\\ \left. \underline{abp^2} \quad \underline{a^2p^2} \quad ab \quad b^2 \quad \underline{bp^2} \quad \underline{p^3} \quad ap \quad bp \quad p^2 \quad a \quad b \quad p \quad 1\right]^{\mathrm T},
    \end{multline}
    where we underlined the new monomials (columns) of matrix~$B_1$.

\item The polynomials corresponding to the last three rows of~$\tilde B_1$ are divisible by~$p$. Matrix~$B_2$ of size $13\times 32$ is obtained as follows. We append~6 new rows to~$\tilde B_1$, which are $(\tilde B_1)_i/p$, $a(\tilde B_1)_i/p$ and $b(\tilde B_1)_i/p$ for $i = 6, 7$. Monomial vector $y_2 = y_1$.

\item Matrix~$B_3$ of size $19\times 32$ is obtained from~$\tilde B_2$ by appending 6 new rows: $a(\tilde B_2)_i$, $b(\tilde B_2)_i$ and $p(\tilde B_2)_i$ for $i = 12, 13$. Monomial vector $y_3 = y_1$.

\item The last row of~$\tilde B_3$ corresponds to a 2nd degree polynomial. Thus we proceed with the polynomials of degree up to~$3$. We eliminate from~$\tilde B_3$ rows and columns corresponding to all 4th degree polynomials and monomials respectively. Matrix~$B_4$ of size $11\times 20$ is obtained from~$\tilde B_3$ as follows. We hold the rows of~$\tilde B_3$ with numbers 4, 10, 11, 12, 13, 16, 17, 19, and append 3 new rows: $a(\tilde B_3)_{19}$, $b(\tilde B_3)_{19}$ and $p(\tilde B_3)_{19}$. Monomial vector
    \begin{multline}
    y_4 = \left[a^2b \quad a^3 \quad ab^2 \quad b^3 \quad a^2p \quad ap^2 \quad abp \quad b^2p \right.\\ \left. a^2 \quad ab \quad b^2 \quad bp^2 \quad p^3 \quad ap \quad bp \quad p^2 \quad a \quad b \quad p \quad 1\right]^{\mathrm T}.
    \end{multline}

\item Finally, matrix~$B_5$ of size $14\times 20$ is obtained from $\tilde B_4$ by appending 3 new rows: $a(\tilde B_4)_{11}$, $b(\tilde B_4)_{11}$ and $p(\tilde B_4)_{11}$. Monomial vector $y_5 = y_4$.
\end{itemize}

The last six rows of matrix $\tilde B_5$ constitute the (reduced) Gr\"{o}bner basis of ideal~$J'$ w.r.t. the graded reverse lexicographic order with $a > b > p$.

\subsection{Internal and external parameters}

Given the Gr\"{o}bner basis of~$J'$, the $6\times 6$ action matrix~$M_p$ for multiplication by~$p$ in the quotient ring $\mathbb C[a, b, p]/J'$ can be easily constructed as follows. We denote by~$C$ the $6\times 6$ right lower submatrix of~$\tilde B_5$. Then the first three rows of~$M_p$ are the last three rows of $(-C)$. The rest of~$M_p$ consists of almost all zeros except
\begin{equation}
(M_p)_{41} = (M_p)_{52} = (M_p)_{65} = 1.
\end{equation}

The six solutions are then found from the eigenvectors of matrix~$M_p$, see~\cite{CLS} for details. Complex solutions and the ones with $p < 0$ are excluded.

Having found the calibration matrix~$K$, we compute the essential matrix~$E$ from formula~\eqref{eq:FtoE} and then the externals~$R$ and~$t$ using the standard procedure, see e.g.~\cite{Nister}. Note that, due to Proposition~\ref{prop:constrEtau2} and the \emph{cheirality constraint}~\cite{HZ,Nister}, the trace of the estimated matrix~$R$ must equal~$\tau$.

Finally, the denormalized entities (see subsection~\ref{ssec:normal} and the definition of matrix~$S$) are found as follows:
\begin{itemize}
\item fundamental matrix is $S^{\mathrm T}FS$;
\item calibration matrix is $S^{-1}K$;
\item essential matrix is unchanged, as $K^{\mathrm T}S^{-\mathrm T}S^{\mathrm T}FSS^{-1}K = K^{\mathrm T}FK \sim E$;
\item externals~$R$ and~$t$ are also unchanged.
\end{itemize}

\section{Experiments on Synthetic Data}
\label{sec:synth}

In this section we test the algorithm on synthetic image data. The default data setup is given in the following table:
\begin{center}
\smallskip\begin{tabular}{|c|c|}
\hline
Distance to the scene & 1\\\hline
Scene depth & 0.5\\\hline
Baseline length & 0.1\\\hline
Image dimensions & $1280 \times 720$ \\\hline
\end{tabular}\smallskip
\end{center}

\subsection{Numerical accuracy}

First we verify the numerical accuracy of our method. The numerical error is defined as the minimal relative error in the calibration matrix, that is
\begin{equation}
\min\limits_i\left(\frac{\|K_i - K_\mathrm{gt}\|}{\|K_\mathrm{gt}\|} \right).
\end{equation}
Here $\|\cdot\|$ stands for the Frobenius norm, $i$ counts all real solutions, and
\begin{equation}
K_\mathrm{gt} = \begin{bmatrix}1000 & 0 & 640 \\ 0 & 1000 & 360 \\ 0 & 0 & 1 \end{bmatrix}
\end{equation}
is the ground truth calibration matrix. The numerical error distribution is reported in Fig.~\ref{fig:nerror}.

\begin{figure}
\centering
\includegraphics[width=0.5\hsize]{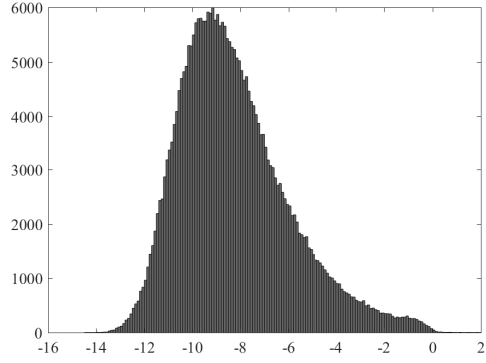}
\caption{$\log_{10}$ of numerical error for noise free data. The median error is $2.5\times 10^{-9}$}
\label{fig:nerror}
\end{figure}

\subsection{Number of solutions}
\label{ssec:nsols}

\begin{figure}
\centering
\includegraphics[width=1.0\hsize]{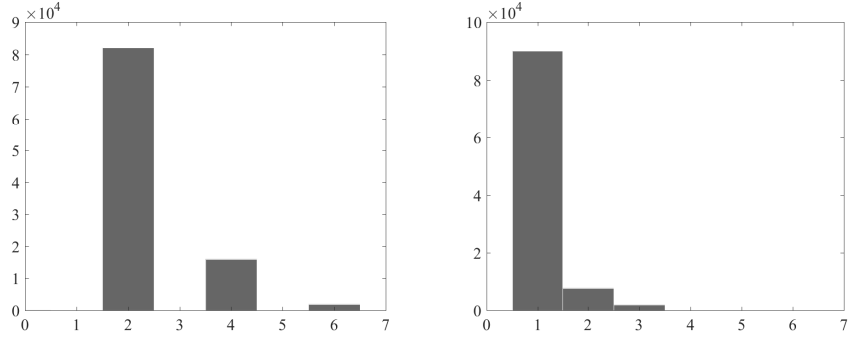}
\caption{The number of real (left) and feasible (right) solutions for the internal calibration. Noise free data}
\label{fig:nsols}
\end{figure}

In general, the algorithm outputs six solutions for the calibration matrix, counting both real and complex ones. The number of real solutions is usually two or four. The number of real solutions with positive~$p$ (we call such solutions feasible) is equal to one in most cases. The corresponding distributions are demonstrated in Fig.~\ref{fig:nsols}.

\subsection{Behaviour under noise}

To evaluate the robustness of our solver, we add two types of noise to the initial data. First, the image noise which is modelled as zero-mean, Gaussian distributed with a standard deviation varying from 0 to 1 pixel in a $1280 \times 720$ image. Second, the angle noise resulting from inaccurately found relative rotation angle~$\theta$. In practice, $\theta$ is computed by integrating the angular velocity measurements obtained from a 3d gyroscope. Therefore, a realistic model for the angle noise is quite complicated and depends on a number of factors such as the value of~$\theta$, the measurement rates, the inertial sensor noise model, etc. In a very simplified manner, the angle noise can be modelled as $\theta s$~\cite{LHLP}, where~$s$ has the Gaussian distribution with zero mean and standard deviation~$\sigma$. In our experiments~$\sigma$ ranges from~$0$ to~$0.09$.

\begin{figure}
\centering
\includegraphics[width=1.0\hsize]{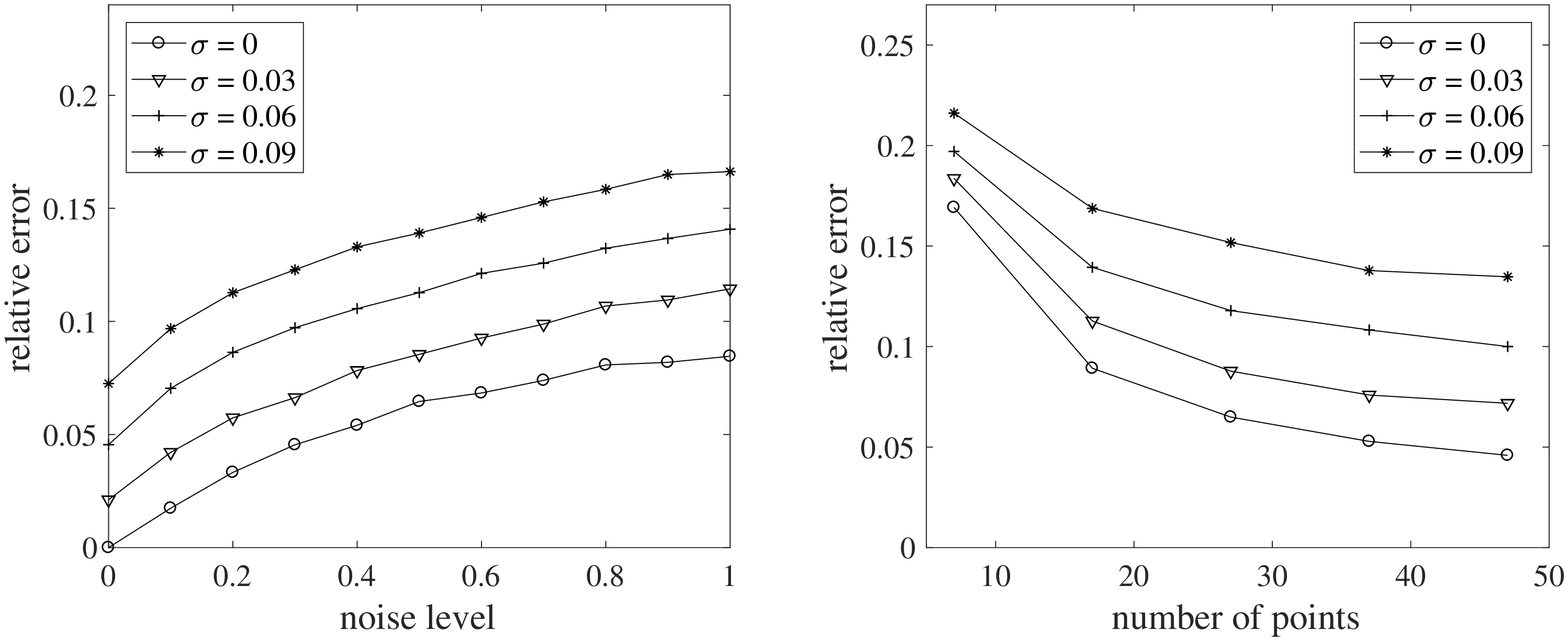}
\caption{The relative error in the calibration matrix at varying level of noise (left) and at varying number of points (right). The number of image points on the left figure is $N = 20$. The standard deviation for the image noise on the right figure is 1~pixel}
\label{fig:K}
\end{figure}

Fig.~\ref{fig:K} demonstrates the behaviour of the algorithm under increasing the image and angle noise. Here and below each point on the diagram is a median of $10^4$ trials.

\subsection{Comparison with the existing solvers}

We compare our algorithm for the minimal number of points ($N = 7$) with the 6-point solver from~\cite{SNKS} and the 5-point solver from~\cite{SEN06}.

Fig.~\ref{fig:f} depicts the relative focal length error $|f - 1000|/1000$ at varying levels of image noise. Here and below $\alpha = 0.1$ means a $10\%$-miscalibration in the parameters~$a$ and~$b$, i.e. the data was generated using the ground truth calibration matrix~$K_\mathrm{gt}$, whereas the solutions were found assuming that
\begin{equation}
K = \begin{bmatrix}1000 & 0 & (1 + \alpha)640 \\ 0 & 1000 & (1 + \alpha)360 \\ 0 & 0 & 1 \end{bmatrix}.
\end{equation}

\begin{figure}
\centering
\includegraphics[width=0.5\hsize]{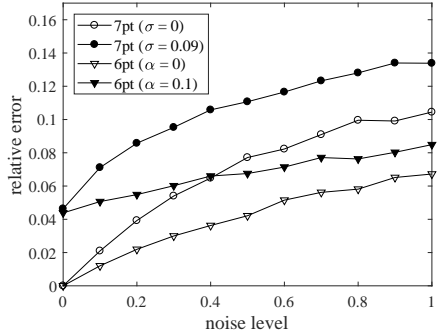}
\caption{The relative focal length error against noise standard deviation in pixels}
\label{fig:f}
\end{figure}

\begin{figure}
\centering
\includegraphics[width=1.0\hsize]{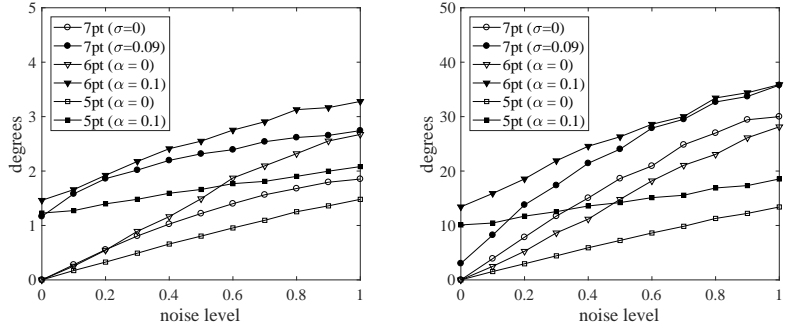}
\caption{The rotational (left) and translational (right) errors in degrees against noise standard deviation in pixels}
\label{fig:RT}
\end{figure}

In Fig.~\ref{fig:RT}, the rotational and translational errors of the algorithms are reported. As it can be seen, for realistic levels of image noise, the 5-point algorithm expectedly outperforms as our solution as the 6-point solver. However, it is worth emphasizing that our solution is more suitable for the internal self-calibration of a camera rather than for its pose estimation. Once the self-calibration is done, it is more efficient to switch to the 5- or even 4-point solvers from~\cite{LHLP,Nister}.

We also compared the speed of the algorithms. The average running times over $10^5$ trials are $2.0$~ms (our 7pt), $1.6$~ms (6pt) and $0.4$~ms (5pt) on a system with 2.3~GHz processor. The most expensive step of our algorithm is the computation of the five reduced row echelon forms in sequence~\eqref{eq:seq}.

\section{Experiments on Real Data}
\label{sec:real}

In this section we validate the algorithm by using the publicly available EuRoC dataset~\cite{EuRoC}. This dataset contains sequences of the data recorded from an IMU and two cameras on board a micro-aerial vehicle and also a ground truth. Specifically, we used the ``Machine Hall~01'' dataset (easy conditions) and only the images taken by camera~``0''. The image size is $752 \times 480$ (WVGA) and the ground truth calibration matrix is
\begin{equation}
K_\mathrm{gt} = \begin{bmatrix}458.654 & 0 & 367.215 \\ 0 & 457.296 & 248.375 \\ 0 & 0 & 1 \end{bmatrix}.
\end{equation}

\begin{figure}
\centering
\includegraphics[width=1.0\hsize]{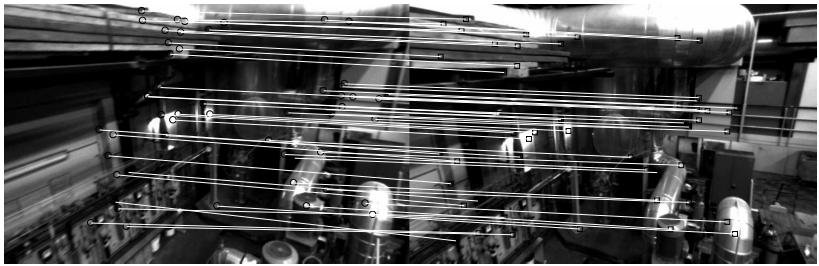}
\caption{The pair of undistorted images with the time stamps ``1403636646263555584'' and ``1403636646613555456'' from the EuRoC dataset and the matched points}
\label{fig:images}
\end{figure}

The sequence of 68 image pairs was derived from the dataset and the algorithm was applied to every image pair, see example in Fig.~\ref{fig:images}. Here it is necessary to make a few remarks.
\begin{itemize}
\item Since the algorithm assumes the pinhole camera model, every image was first undistorted using the ground truth parameters.

\item As it was mentioned in Subsection~\ref{ssec:nsols}, the feasible solution is almost always unique. However in rare cases multiple solutions are possible. To reduce the probability of getting such solutions we additionally assumed that the principal point is sufficiently close to the geometric image center. More precisely, the solutions with
\begin{equation}
(a, b) \not\in \{|x - 376| < 50, |y - 240| < 50\}
\end{equation}
were marked as unfeasible and hence discarded. This condition almost guarantees that the algorithm outputs at most one solution.

\item Given the readings of a triple-axis gyroscope the relative rotation angle can be computed as follows. The gyroscope reading at time~$\xi_i$ is an angular rate 3-vector~$w_i$. Let $\Delta \xi_i = \xi_i - \xi_{i - 1}$, where $i = 1, \ldots, n$. Then the relative rotation matrix~$R_n$ between the $0$th and $n$th frames is approximately found from the recursion
    \begin{equation}
    R_i = \exp([w_i]_\times \Delta \xi_i) R_{i - 1},
    \end{equation}
    where $R_0 = I$ and the matrix exponential is computed by the Rodrigues formula
    \begin{equation}
    \exp([v]_\times) = I + \frac{\sin\|v\|}{\|v\|}\, [v]_\times + \frac{1 - \cos\|v\|}{\|v\|^2}\, [v]^2_\times.
    \end{equation}
    The relative rotation angle is then derived from $\tr R_n$.

\item The image pairs with the relative rotation angle less than~$5$ degrees were discarded, since the motion in this case is close to a pure translation and self-calibration becomes unstable.
\end{itemize}

The estimated internal parameters for each image pair are shown in Fig.~\ref{fig:results}. The calibration matrix averaged over the entire sequence is given by
\begin{equation}
K = \begin{bmatrix}457.574 & 0 & 366.040 \\ 0 & 457.574 & 243.616 \\ 0 & 0 & 1 \end{bmatrix}.
\end{equation}
Hence the relative error in the calibration matrix is about~$0.6~\%$.

\begin{figure}
\centering
\includegraphics[width=1.0\hsize]{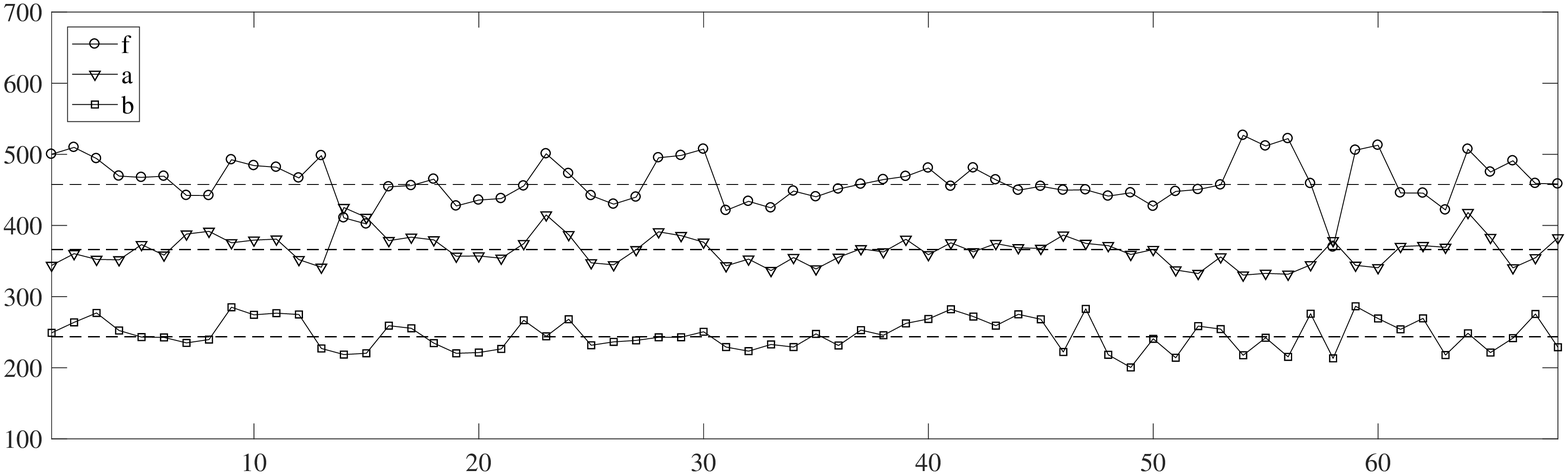}
\caption{The estimated focal length and principal point coordinates for the 68 image pairs. The dashed horizontal lines represent the average values}
\label{fig:results}
\end{figure}

\section{Conclusions}
\label{sec:disc}

We have presented a new practical solution to the problem of self-calibration of a camera with Euclidean image plane. The solution operates with at least seven point correspondences in two views and also with the known value of the relative rotation angle.

Our method is based on a novel quadratic constraint on the essential matrix, see Eq.~\eqref{eq:constrEtau}. We expect that there could be other applications of that constraint. In particular, it can be used to obtain an alternative solution of the problem from paper~\cite{LHLP}. The investigation of that possibility is left for further work.

We have validated the solution in a series of experiments on synthetic and real data. Under the assumption of generic camera motion and points configuration, it is shown that the algorithm is numerically stable and demonstrates a good performance in the presence of noise. It is also shown that the algorithm is fast enough and in most cases it produces a unique feasible solution for camera calibration.

\bibliographystyle{amsplain}
\bibliography{self2v-gyro}

\end{document}